\documentclass{article}

\usepackage[preprint]{neurips_2026}

\usepackage[utf8]{inputenc} % allow utf-8 input
\usepackage[T1]{fontenc}    % use 8-bit T1 fonts
\usepackage{url}            % simple URL typesetting
\usepackage{booktabs}       % professional-quality tables
\usepackage{amsfonts}       % blackboard math symbols
\usepackage{nicefrac}       % compact symbols for 1/2, etc.
\usepackage{microtype}      % microtypography
\usepackage{xcolor}         % colors
\usepackage{thmtools}       % for restatable environment
\usepackage{microtype}

\usepackage{graphicx}
\usepackage{amsmath}
\usepackage{amssymb}
\usepackage{subfigure}
\usepackage{caption} 
\usepackage{natbib}
\usepackage{enumitem}

\usepackage{booktabs} % for professional tables
\usepackage{amsmath}
\usepackage{amssymb}
\usepackage{mathtools}
\usepackage{amsthm}

\theoremstyle{plain}

\theoremstyle{definition}

\theoremstyle{remark}

\usepackage{listings}          
\usepackage{algorithm}  
\usepackage{algorithmicx}  
\usepackage{algpseudocode}

\algnewcommand{\LeftComment}[1]{\Statex \(\triangleright\) #1}

\usepackage{array}
\usepackage{amsmath}
\usepackage{amssymb}
\usepackage{mathtools}
\usepackage{amsthm}
\usepackage{arydshln}
\usepackage{adjustbox}
\usepackage{enumitem}

\usepackage{listings}           
\usepackage{tcolorbox}
\tcbuselibrary{listings,skins}        
\newtcolorbox{promptspecial}[1][]{
  colback      = gray!10,        
  colframe     = gray!50!black,  
  fontupper    = \ttfamily,      
  nobeforeafter,
  title        = Prompt Template,
  #1
}

\newtcolorbox{prompttext}[1][]{
  colback=gray!10,        
  colframe=gray!50!black, 
  nobeforeafter,
  title=Prompt Template,
  #1
}

\newtcolorbox{outputcase}[1][]{
  colback=gray!10,        
  colframe=gray!50!black, 
  nobeforeafter,
  title=Output Case,
  #1
}

\usepackage[utf8]{inputenc} % allow utf-8 input
\usepackage[T1]{fontenc}    % use 8-bit T1 fonts
\usepackage{url}            % simple URL typesetting
\usepackage{booktabs}       % professional-quality tables
\usepackage{amsfonts}       % blackboard math symbols
\usepackage{nicefrac}       % compact symbols for 1/2, etc.
\usepackage{microtype}      % microtypography
\usepackage{xcolor}         % colors

\usepackage{natbib}
\usepackage[colorlinks=true, citecolor=blue]{hyperref}

\usepackage[capitalize,noabbrev]{cleveref}

\usepackage[capitalize,noabbrev]{cleveref}
\title{FABSVer: Faster Training and Better Self-Verification for LLM Mathematical Reasoning}

\author{Haihui Pan, Junwei Bao, Hongfei Jiang, Yang Song \\
Zuoyebang Education Technology\\
}

\begin{document}

\maketitle

\begin{abstract}
While large language models have made significant progress in mathematical reasoning, they remain unreliable at judging the correctness of their own solutions. Existing approaches that equip models with self-verification typically treat solution generation and verification as two separate tasks, leading to substantially increased training time.  In this paper, we propose \textbf{FABSVer}, which fuses these two tasks into a single generation pass, dramatically reducing training overhead while jointly optimizing both capabilities. We further identify a convergence bottleneck both theoretically and empirically: as training progresses, the reward reaches a plateau because the policy is constrained by a fixed reference model. To overcome this, we introduce Dynamic Reference Model Update (DRMU), which raises the reward ceiling and enables sustained reward growth. Extensive experiments on math benchmarks demonstrate that FABSVer achieves superior self-verification and reasoning performance across three model scales, while requiring only 51\%--71\% of the training time of existing methods.  Analysis further reveals distinct learning phases in how models acquire self-verification, and that the gap between verify and answer rewards shrinks noticeably as model size increases.

\end{abstract}

\section{Introduction}

Large Language Models (LLMs) have become a foundational paradigm in artificial intelligence, demonstrating remarkable capabilities across a wide range of natural language processing tasks. Through scaling model size, data, and compute, modern LLMs exhibit emergent abilities in code generation, commonsense reasoning, and mathematical problem-solving \citep{r1, o1, qw2.5}. Among these, mathematical reasoning serves as a critical benchmark for evaluating logical deduction and multi-step inference, as it demands not only factual knowledge but also rigorous step-by-step reasoning.

In recent years, LLMs have advanced substantially in mathematical reasoning, largely driven by Reinforcement Learning with Verifiable Rewards (RLVR) \citep{r1, simplerl-zoo, o1, hu2025open}. By leveraging objective, rule-based rewards derived solely from final answer correctness, RLVR enables scalable training and elicits sophisticated reasoning processes.

However, despite their improved problem-solving accuracy, LLMs trained via RLVR suffer from a critical limitation: overconfidence in their own outputs. As noted in recent studies, even when a generated solution is fundamentally incorrect, the model frequently perceives it as correct \citep{deepseekmath-v2, rise, sareen2025putting, zhang2025incentivizing}. This verification gap motivates the need for self-verification—the ability of a model to accurately judge the correctness of its own generated sequences.

The importance of self-verification extends well beyond improved accuracy. From a trustworthiness standpoint, LLM reasoning remains largely opaque; self-verification offers a window into the model's confidence, helping users decide when to rely on its outputs and enabling models to flag their own potential errors. Most fundamentally, from a comprehension perspective, solving a problem without the ability to judge the correctness of a solution reflects, at best, a superficial understanding. Just as true mastery entails evaluating both one's own answers and those of others, genuine reasoning requires the integration of generation and verification.

Recent work on intrinsic self-verification, such as \textsc{RISE} \citep{rise}, jointly optimizes solution generation and verification during online reinforcement learning. However, these approaches treat generation and verification as two independent forward passes. Given $N$ questions, with $k$ solutions generated per question and $m$ verification judgments per solution, the total number of forward passes inflates from $N \cdot k$ to $N \cdot k \cdot (1 + m)$, multiplicatively increasing both training data and inference time. This overhead makes large-scale self-verification training prohibitively expensive.

Beyond training efficiency, a more fundamental challenge stems from the optimization dynamics of RLVR itself. In standard frameworks such as PPO and GRPO \citep{ppo, deepseekmath-v2}, a fixed reference model regularizes policy updates via a KL divergence penalty. While this stabilizes training, it imposes an implicit upper bound on policy improvement: as we theoretically demonstrate, once the KL divergence ceases to grow, the achievable reward is effectively capped. This bottleneck is particularly detrimental for self-verification training, which requires sustained exploration over extended horizons.

To address these challenges, we propose \textbf{FABSVer} (\underline{F}aster Training \underline{a}nd \underline{B}etter \underline{S}elf-\underline{Ver}ification), built upon two key innovations: (1) \textbf{Task Fusion}, which merges solution generation and self-verification into a single forward pass, eliminating the multiplicative overhead; and (2) \textbf{Dynamic Reference Model Update (DRMU)}, which periodically resets the reference model to the current policy, releasing the KL constraint and enabling sustained reward growth. We provide theoretical justification for DRMU by deriving an upper bound relating the expected reward to the KL divergence.

We validate FABSVer on three model scales (1.5B, 3B, and 7B) across multiple math benchmarks. The results, detailed in Section~\ref{sec: experiment}, show that FABSVer consistently outperforms existing methods in both self-verification and reasoning, while requiring substantially less training time. Beyond performance gains, our analysis uncovers new insights into the training dynamics of self-verification.

Our contributions are summarized as follows:
\begin{itemize}
    \item We propose \textbf{FABSVer}, a task-fused training framework that unifies solution generation and self-verification into a single forward pass, eliminating the multiplicative overhead of existing approaches while enabling synergistic optimization of both capabilities.
    \item We introduce \textbf{DRMU}, a simple yet theoretically grounded mechanism that overcomes the convergence bottleneck caused by a fixed reference model, enabling sustained reward growth with theoretical guarantees.
    \item Through extensive experiments across three model scales, we demonstrate that FABSVer achieves superior performance at substantially lower training cost, and we uncover distinct learning phases and scaling effects in self-verification acquisition.
\end{itemize}

\section{Related work}

\paragraph{RLVR for Reasoning.}  Reinforcement Learning with Verifiable Rewards (RLVR) has recently emerged as a powerful paradigm for enhancing LLM reasoning capabilities \citep{r1,simplerl-zoo, o1, hu2025open}. 
Unlike previous RL approaches that rely on either human preference annotations \citep{rlhf, dpo} or complex process reward models that require fine-grained supervision \citep{math500-2, setlur2024rewarding}, RLVR leverages objective, rule-based rewards derived solely from final answer correctness, such as exact match with ground truth or passing unit tests.
This key advantage—providing feedback based only on the final output without requiring complete output annotations or dense signals from process supervision—enables scalable and efficient training for reasoning tasks. Recent works \citep{r1,simplerl-zoo, o1, hu2025open} have demonstrated that such outcome-based RL can effectively elicit sophisticated reasoning processes in LLMs, achieving substantial improvements on mathematical and programming benchmarks.

\paragraph{Self-Verification for LLM.} Self-verification—the ability of a model to evaluate the correctness of its own outputs—has been recognized as a crucial capability for reliable LLM reasoning \citep{deepseekmath-v2, rise, sareen2025putting, zhang2025incentivizing, lin2025learning}. While RLVR effectively improves a model's problem-solving accuracy through outcome-based rewards, recent studies have identified a notable limitation: models trained with RLVR often lack robust self-verification capabilities \citep{deepseekmath-v2, rise, sareen2025putting, zhang2025incentivizing}. Specifically, these models tend to exhibit overconfidence, incorrectly judging their own erroneous solutions as correct. This observation has motivated a growing body of research aimed at equipping models with both problem-solving and self-verification abilities simultaneously \citep{deepseekmath-v2, rise, zhang2025incentivizing}. By jointly optimizing for solution generation and correctness verification, these approaches seek to address the verification gap and further enhance the reliability of LLM reasoning systems.

\section{Preliminaries}

\subsection{Reinforcement Learning for LLMs}

Reinforcement Learning (RL) has emerged as a crucial technique for aligning LLMs with human preferences and specific task objectives. In the context of LLMs, the RL framework is adapted as follows: the policy $\pi_\theta$ is the language model itself, parameterized by $\theta$, which generates text autoregressively. At each generation step $t$, given a prompt $\textbf{x}$ and previously generated tokens $\textbf{y}_{<t}$, the model produces a next token $y_t$ according to $\pi_\theta(y_t|\textbf{x}, \textbf{y}_{<t})$. The complete response $\textbf{y} = (y_1, y_2, ..., y_{|\textbf{y}|})$ forms a trajectory. After generating the complete response, a reward model or task-specific metric assigns a scalar reward $r$ reflecting the overall quality of the response.

\paragraph{Policy gradient (PG).} PG methods \citep{pg} provide a way to optimize the language model directly by estimating the gradient of the expected reward. For LLMs, the standard policy gradient objective can be written as:
\begin{align}
\mathcal{J}_{PG}(\theta) = \mathbb{E}_{\textbf{x} \sim \mathcal{D}, \textbf{y} \sim \pi_\theta(\cdot|\textbf{x})} \left[ r(\textbf{x}, \textbf{y}) \right],
\end{align}
where $\mathcal{D}$ is the input distribution and $r(\textbf{x}, \textbf{y})$ is the reward for response $\textbf{y}$ given input $\textbf{x}$. 

\paragraph{Proximal Policy Optimization (PPO).}
PPO  \citep{ppo} was originally designed to improve training stability by restricting policy updates to a trust region. In the context of LLMs, the objective of PPO is typically formulated as maximizing a clipped surrogate objective while incorporating a KL-divergence penalty to keep the learned policy close to a reference model \( \pi_{\text{ref}} \):
\begin{align}
\mathcal{J}_{\text{PPO}}(\theta) = \mathbb{E}_{\textbf{x} \sim \mathcal{D}, \textbf{y} \sim \pi_{\text{old}}(\cdot|\textbf{x})} \left[ \frac{1}{|\textbf{y}|} \sum_{t=1}^{|\textbf{y}|} \min \left( r_t(\theta) \hat{A}_t, \text{clip}(r_t(\theta), 1-\epsilon, 1+\epsilon) \hat{A}_t \right) - \beta \text{KL}[\pi_{\theta} || \pi_{\text{ref}}] \right],
\end{align} 
where \( r_t(\theta) = \frac{\pi_{\theta}(y_t|\textbf{x}, \textbf{y}_{<t})}{\pi_{\text{old}}(y_t|\textbf{x}, \textbf{y}_{<t})} \) is the importance sampling ratio, \( \hat{A}_t \) is the advantage estimate, and \( \beta \) is a coefficient controlling the strength of the KL penalty. PPO relies on a value model to compute advantages, which introduces significant memory and computational overhead during training.

\paragraph{Group Relative Policy Optimization (GRPO).} To eliminate the need for an expensive value model, GRPO \citep{grpo} introduces a novel group-based advantage estimation method. For the same input \(\textbf{x}\), GRPO first samples a group of \(G\) outputs \(\{ \textbf{y}_1, \textbf{y}_2, \dots, \textbf{y}_G \}\) from the old policy \( \pi_{\text{old}} \). Instead of learning a value model, it estimates the advantage by comparing rewards within this group. For each output \(\textbf{y}_i\), we obtain its reward score \( r_i = r(\textbf{x}, \textbf{y}_i) \) from the reward model. The advantage \( \hat{A}_i \) for the \(i\)-th output is defined as the reward normalized by the mean and standard deviation of the group. The optimization objective of GRPO is then formulated as:
\begin{align}
\label{eq:GRPO}
\mathcal{J}_{\text{GRPO}}(\theta) = & \mathbb{E}_{\textbf{x} \sim \mathcal{D}, \{\textbf{y}_i\}_{i=1}^{G} \sim \pi_{\text{old}}(\cdot|\textbf{x})} \nonumber
\\
& \left[ \frac{1}{G} \sum_{i=1}^{G}  \left( \min \left( r_{i}(\theta) \hat{A}_i, \text{clip}(r_{i}(\theta), 1-\epsilon, 1+\epsilon) \hat{A}_i \right) - \beta \text{KL}[\pi_{\theta} || \pi_{\text{ref}}] \right) \right],
\end{align}
where \( r_i(\theta) = \frac{\pi_{\theta}(\textbf{y}|\textbf{x})}{\pi_{\text{old}}(\textbf{y}|\textbf{x})} \) is the importance sampling ratio and $\hat{A}_i = \frac{r_i - \text{mean}(r_1,\cdots, r_G)}{\text{std}(r_1,\cdots, r_G)}$.

\subsection{Self-Verification for LLMs}

Recent studies indicate that  LLMs trained via RVLR often suffer from overconfidence regarding their generated outputs $\textbf{y}$. Even when the generated solution $\textbf{y}$ is fundamentally incorrect, the model frequently perceives it as correct \citep{deepseekmath-v2, rise, sareen2025putting, zhang2025incentivizing}. The primary goal of self-verification is to endow the policy $\pi_\theta$ with the capability to accurately judge the correctness of its own generated sequence $\textbf{y}$. 

To develop this self-verification capability, existing approaches typically joint-train two tasks during the reinforcement learning phase: solution generation and solution verification \citep{deepseekmath-v2, rise, sareen2025putting, zhang2025incentivizing}. Specifically, given an input question $\textbf{x}$, the policy $\pi_\theta$ first generates a solution $\textbf{y}$. Subsequently, conditioned on $\textbf{x}$, $\textbf{y}$, and a specialized self-verification prompt $I_v$, the policy produces a verification outcome $\textbf{v}$. For verifiable mathematical problems, the reward for solution generation can be straightforwardly derived via rule-based evaluation, checking whether the generated answer $\textbf{y}$ matches the ground-truth $\textbf{y}^{*}$. Likewise, the reward signal for the verification task can also be determined using rule-based logic based on the triplet $(\textbf{v}, \textbf{y}, \textbf{y}^{*})$. For instance, when the verification output $\textbf{v}$ asserts that the generated solution is correct, a true reward is obtained by checking whether $\textbf{y}$ is indeed consistent with $\textbf{y}^{*}$. 

Therefore, the overall optimization objective for these methods integrates the RL objectives of both generation and verification, which can be formulated as:
\begin{align}
\mathcal{J}_{\text{SV}}(\theta) = \mathbb{E}_{(\textbf{x}, \textbf{y}^{*}) \sim \mathcal{D}, \textbf{y} \sim \pi_\theta(\cdot|\textbf{x})} \left[ r_{\text{gen}}(\textbf{y}, \textbf{y}^{*}) +  \mathbb{E}_{\textbf{v} \sim \pi_\theta(\cdot|\textbf{x}, \textbf{y}, I_v)} \left[ r_{\text{ver}}(\textbf{v}, \textbf{y}, \textbf{y}^{*}) \right] \right],
\end{align}
where $\mathcal{D}$ denotes the dataset containing questions $\textbf{x}$ and their corresponding ground-truth answers $\textbf{y}^{*}$; $r_{\text{gen}}$ and $r_{\text{ver}}$ are the rule-based reward functions for the solution generation and verification tasks, respectively.

The fundamental difference between \citep{rise} and \citep{zhang2025incentivizing} lies in the formalization of the verification prompt $I_v$. While \citep{rise} treats the self-verification query as a completely independent, standalone prompt, \citep{zhang2025incentivizing} structures it as a follow-up query within a multi-turn conversation (i.e., a second-turn request). The $I_v$ of these two methods are shown in \cref{fig:rise_ver_prompt} and \cref{fig:self_ver_prompt}, respectively.
\citep{deepseekmath-v2} uses human-annotated data to train Verification and Meta-Verification, which provide feedback for solution generation and solution verification. Subsequently, RL is used to train the policy's ability to generate solutions and its self-verification capability.

\section{Faster Training and Better Self-Verification (FABSVer)}

Existing methods for self-verification treat solution generation and verification as two separate tasks, performing independent forward passes for each during online RL training. Given $N$ questions with $k$ solutions per question and $m$ verification judgments per solution, the total number of forward passes expands from $N \cdot k$ to $N \cdot k \cdot (1+m)$, multiplicatively increasing both training data and inference time. This prohibitive overhead stems from a fundamental design choice: the decoupling of generation and verification.

\paragraph{Motivation and Core Idea.} To address this, we propose \textbf{FABSVer}, which \textbf{fuses solution generation and self-verification into a single forward pass}, eliminating the multiplicative overhead. To further address the convergence bottleneck caused by the fixed reference model, we introduce \textbf{DRMU}, which periodically resets the reference model to the current policy, releasing the KL constraint and enabling sustained reward growth.

Due to space constraints, we defer the algorithm overview to Appendix~\ref{sec:alg-overview}. The overall procedure of FABSVer is summarized in Algorithm~\ref{alg:fabsver}: each generation produces both the solution and the verification result in a single forward pass, optimizes with the combined reward \(r_{\text{combine}}\), and periodically updates the reference model.

\subsection{Task Fusion and Reward Design}

\textbf{Task Fusion.} To reduce training overhead, we propose to output both the solution and the self-verification result in a single generation. Specifically, given an input \(\mathbf{x}\), the model \(\pi_\theta\) generates a solution \(\mathbf{y}\) and a verification result \(\mathbf{v}\) in one forward pass. The prompt template used for training and inference is shown in \cref{fig:prompt}, where the model is instructed to output the reasoning process, answer, and verification result within \(<\texttt{think}>\), \(<\texttt{answer}>\), and \(<\texttt{verify}>\) tags, respectively. Compared with existing methods, the training data size does not incur additional increase, and no second forward pass is required for verification, thereby significantly reducing training time and computational cost.

\textbf{Reward Design.} Under the single-generation framework, we redesign the reward function. We introduce three binary reward signals based on the format, the solution, and the verification result: \(r_{\text{format}}(\mathbf{y}) \in \{0,1\}\) indicates whether the output format is correct, \(r_{\text{gen}}(\mathbf{y}, \mathbf{y}^*) \in \{0,1\}\) indicates whether the generated solution \(\mathbf{y}\) matches the ground truth \(\mathbf{y}^*\), and \(r_{\text{ver}}(\mathbf{y}, \mathbf{v}, \mathbf{y}^*) \in \{0,1\}\) indicates whether the verification result \(\mathbf{v}\) correctly judges the solution. 
The final combined reward \(r_{\text{combine}}(\mathbf{y}, \mathbf{v}, \mathbf{y}^*)\) is then defined as:
\begin{align}
\label{eq:combine}
r_{\text{combine}}(\mathbf{y}, \mathbf{v}, \mathbf{y}^*) = 
\begin{cases} 
1.0, & \text{if } r_{\text{format}}(\mathbf{y}) = 1 \text{ and } r_{\text{gen}}(\mathbf{y}, \mathbf{y}^*) = 1 \text{ and } r_{\text{ver}}(\mathbf{y}, \mathbf{v}, \mathbf{y}^*) = 1 \\[4pt]
0.5, & \text{if } r_{\text{format}}(\mathbf{y}) = 1 \text{ and } r_{\text{gen}}(\mathbf{y}, \mathbf{y}^*) = 1 \text{ and } r_{\text{ver}}(\mathbf{y}, \mathbf{v}, \mathbf{y}^*) = 0 \\[4pt]
0.25, & \text{if } r_{\text{format}}(\mathbf{y}) = 1 \text{ and } r_{\text{gen}}(\mathbf{y}, \mathbf{y}^*) = 0 \text{ and } r_{\text{ver}}(\mathbf{y}, \mathbf{v}, \mathbf{y}^*) = 1 \\[4pt]
0, & \text{otherwise}
\end{cases}
\end{align}

\textbf{Design Principles:} The design of \(r_{\text{combine}}(\mathbf{y}, \mathbf{v}, \mathbf{y}^*)\) follows three principles.

\begin{enumerate}
    \item \textbf{Reward for correct solutions is significantly higher than for incorrect ones:} A correct solution with correct verification receives the highest reward (1.0), a correct solution with incorrect verification receives a medium reward (0.5), while an incorrect solution can receive at most 0.25 (only if the verification is correct). This reflects the consideration that solution generation is more difficult than self-verification.
    
    \item \textbf{Encourage correct self-verification:} Even if the solution is incorrect, a correct verification still receives a small positive reward (0.25) to guide the model to learn accurate judgment of its own output. If both the solution and the verification are incorrect, the reward is zero.
    
    \item \textbf{Both correct solution and correct verification are required for the highest reward:} The model must improve both its generation ability and its self-verification ability simultaneously to achieve optimal performance.
\end{enumerate}

\textbf{Additional Advantages of Task Fusion.} Beyond reducing training overhead, task fusion brings two additional advantages:

\begin{itemize}
    \item \textbf{Flexible output control.} In practice, the \(<\texttt{verify}>\) tag can be treated as an EOS (End of Sequence) token, allowing the model to flexibly decide whether to output the verification result on demand without any architectural modification.
    
    \item \textbf{Joint optimization effect.} In the two-task framework, generation and verification are optimized independently by \(r_{\text{gen}}\) and \(r_{\text{ver}}\) respectively; verification errors do not affect the reward for generation. In contrast, under the fused framework, the combined reward \(r_{\text{combine}}\) depends on both the correctness of the solution and the correctness of the verification. For example, even if the solution is correct, a wrong verification reduces the total reward from 1.0 to 0.5. This coupling mechanism encourages the model to improve its self-verification ability while optimizing generation, as verification errors directly reduce the overall reward, thereby providing a stronger joint learning signal.
\end{itemize}
% \textbf{Additional Advantages of Task Fusion.} Beyond reducing training overhead, task fusion brings two additional advantages: (1) \textbf{Flexible output control:} In practice, the \(<\texttt{verify}>\) tag can be treated as an EOS (End of Sequence) token, allowing flexible decision of whether to output the verification result on demand without modifying the model architecture. (2) \textbf{Joint optimization effect:} In the two-task framework, generation and verification are optimized independently by \(r_{\text{gen}}\) and \(r_{\text{ver}}\) respectively; verification errors do not affect the reward for generation. In contrast, in the fused framework, the combined reward \(r_{\text{combine}}\) depends on both the correctness of the solution and the correctness of the verification. For example, even if the solution is correct, a wrong verification reduces the total reward from 1.0 to 0.5. This coupling mechanism encourages the model to improve its self-verification ability while optimizing generation, as verification errors directly reduce the overall reward, thereby providing a stronger joint learning signal.
\subsection{Dynamic Reference Model Update (DRMU)}

Although task fusion significantly reduces training overhead, during RL training we observe that the reward increases slowly in the later stage, while the KL divergence barely increases. This indicates that the policy update is constrained by the fixed reference model \(\pi_{\text{ref}}\). Intuitively, while the KL divergence provides stability for policy updates, it also acts as a constraint that limits exploration. We further analyze this phenomenon theoretically in \cref{thm:drmu}. The proof of \cref{thm:drmu} can be found in the Appendix ~\ref{proof}.

\begin{restatable}{theorem}{thmOne}
\label{thm:drmu}
Let \(\pi_{\theta}(y|x)\) be the policy to be optimized, and \(\pi_{\text{ref}}(y|x)\) be the reference model. For every reward model \(r(x,y)\) such that \(\mathbb{E}_{\pi_{\text{ref}}(y|x)}[e^{r(x,y)}] < \infty\), then we have:
\[
\mathbb{E}_{\pi_\theta(y|x)}[r(x,y)] \le \mathrm{KL}[\pi_{\theta} \| \pi_{\text{ref}}] + \log \mathbb{E}_{\pi_{\text{ref}}(y|x)}[e^{r(x,y)}]
\]
\end{restatable}
In RL algorithms such as PPO or GRPO, \(\pi_{\text{ref}}\) is fixed, and thus the second term on the right-hand side of the inequality is constant. Consequently, if \(\text{KL}[\pi_\theta \| \pi_{\text{ref}}]\) no longer increases, the upper bound of the reward becomes locked, making further improvement difficult.

To break this limitation, we propose {Dynamic Reference Model Update (DRMU)}. Specifically, at certain epochs or training steps, we replace the reference model \(\pi_{\text{ref}}\) with the current policy \(\pi_\theta\). This operation resets \(\text{KL}[\pi_\theta \| \pi_{\text{ref}}]\) to 0, allowing the policy to regain room to diverge from the reference model (i.e., releasing the KL budget), thereby enabling continued reward optimization. Meanwhile, the updated reference model is closer to the current policy, which helps maintain training stability. In this way, DRMU effectively accelerates the overall convergence process.

\section{Experiment}
\label{sec: experiment}

\subsection{Setup}
\label{sec:setup}

\paragraph{Models, Dataset and Training Details.} We follow the same experimental setup as \citep{rise} for both training data and model selection. We use three Qwen2.5 base models \citep{qw2.5} of varying sizes—1.5B, 3B, and 7B—for our experiments, and MATH-Hard (Level 3–5) \citep{math500} serves as the training set throughout. Due to space constraints, the training details are provided in Appendix~\ref{sec:train-details}.

\paragraph{Benchmarks and Evaluation.} We evaluate our models on five mathematical reasoning benchmarks. MATH500 \citep{math500, math500-2}, a subset of the MATH test set, serves as an in-distribution test to assess the efficacy of learning. To evaluate out-of-distribution generalization, we use OlympiadBench\citep{olympiadbench}, AIME2024, AMC2023 and Minerva \citep{Minerva}. Among these, AIME2024 and Olympiad feature competition-level problems, allowing us to rigorously assess model performance on challenging, multi-step reasoning tasks. Following the evaluation setup in \citep{rise}, we generate 8 responses per problem with a sampling temperature of 1 and adopt Pass@1 \citep{pass-k} accuracy as our primary metric. Specifically, reasoning correctness is determined by the final answer in the model's response, and verification correctness is assessed based on both the model's predicted verification result and its consistency with the gold answer. To ensure reproducibility of the results, we use the same inference seed for all models.

\subsection{Comparative Analysis of Self-Verification and Reasoning Performance}
\label{sec:main_result}

Experimental results on self-verification and reasoning performance are reported in Table~\ref{tab:zero_shot_results}, while training efficiency comparisons are provided in Table~\ref{tab:training_time}.

From Base to FABSVer, the model exhibits a pronounced \textbf{dual-improvement trend} in both self-verification and reasoning capabilities. As shown in Table~\ref{tab:zero_shot_results}, taking Qwen2.5-1.5B as an example, the average self-verification score surges from 20.2\% to 79.8\%, while the average reasoning score rises from 1.1\% to 26.6\%, indicating that \textbf{FABSVer achieves positive synergistic gains in both capabilities without sacrificing reasoning performance for self-verification accuracy}. This dual-improvement trend holds consistently across the 3B and 7B model sizes, demonstrating the strong scalability of the proposed method.

In stark contrast, Zero-RL focuses predominantly on enhancing reasoning performance but yields only marginal improvements in self-verification. For the 1.5B model, Zero-RL attains a reasoning average of 24.0\%, a substantial gain over the Base score of 1.1\%; however, its self-verification average reaches merely 26.8\%, far below the 79.8\% achieved by FABSVer. This \textbf{``one-sided improvement''} phenomenon persists across both the 3B and 7B models. These results indicate that \textbf{reasoning-oriented training alone cannot spontaneously engender high-quality self-verification capability}. FABSVer, by contrast, successfully realizes synergistic gains in both capacities---substantially raising the ceiling of self-verification while maintaining a lead in reasoning performance.

FABSVer also comprehensively outperforms the concurrent work RISE across both dimensions, with demonstrably lower training costs. As detailed in Table~\ref{tab:zero_shot_results}, FABSVer surpasses RISE in self-verification by 5.3\%, 2.7\%, and 4.5\% on the 1.5B, 3B, and 7B models, respectively. In terms of reasoning, FABSVer likewise maintains a consistent advantage over RISE across all three model sizes, confirming that the enhanced self-verification capability does not come at the expense of reasoning performance. Crucially, with respect to training efficiency, Table~\ref{tab:training_time} shows that the training time of FABSVer amounts to only 51\% to 71\% of that required by RISE, representing a \textbf{Pareto improvement of stronger performance with lower computational overhead}. This comprehensive superiority substantiates that \textbf{FABSVer, through a more streamlined training paradigm, achieves synergistic advancement in both error-detection and reasoning capabilities}.

\begin{table*}[t]
\centering
\caption{Comparison of self-verification and reasoning performance between FABSVer and existing methods on various math benchmarks. Results for other methods are obtained from \citep{rise}.}
\label{tab:zero_shot_results}
\resizebox{1.0\linewidth}{!}{
\setlength{\tabcolsep}{3pt}
\begin{tabular}{lrrrrrr|rrrrrr}
\toprule
& \multicolumn{6}{c|}{\textit{\textbf{Self-Verification}}} & \multicolumn{6}{c}{\textit{\textbf{Reasoning}}} \\
\cmidrule(lr){2-7} \cmidrule(lr){8-13}
\textbf{Model} & \textbf{\small MATH} & \textbf{\small Olym.} & \textbf{\small AIME} & \textbf{\small AMC} & \textbf{\small Mine.} & \textbf{Avg.} & \textbf{\small MATH} & \textbf{\small Olym.} & \textbf{\small AIME} & \textbf{\small AMC} & \textbf{\small Mine.} & \textbf{Avg.} \\
\midrule
GPT-4o & 83.4 & 54.4 & 33.3 & 67.5 & 50.4 & 57.8 & 79.0 & 42.5 & 13.3 & 55.0 & 50.0 & 48.0 \\ \midrule
\multicolumn{7}{l|}{\em Qwen2.5-1.5B} \\
$\quad \vdash$ Base & 19.4 & 21.1 & 21.9 & 22.7 & 15.9 & 20.2 & 2.0 & 0.6 & 0.0 & 1.9 & 0.8 & 1.1 \\
$\quad \vdash$ Instruct & 48.8 & 29.6 & 22.1 & 36.5 & 36.9 & 34.8 & 37.5 & 11.7 & 0.8 & 19.4 & 8.3 & 15.5 \\
$\quad \vdash$ SFT & 19.0 & 10.9 & 5.8 & 12.3 & 10.5 & 11.7 & 10.1 & 2.0 & 0.0 & 4.1 & 1.8 & 3.6 \\
$\quad \vdash$ Zero-RL & 54.1 & 23.0 & 5.0 & 30.7 & 21.0 & 26.8 & 55.3 & 19.5 & 2.1 & 25.9 & {17.4} & 24.0 \\
$\quad \vdash$ {RISE-1.5B} & {75.9} & {74.9} & {85.0} & {70.6} & \textbf{66.0} & {74.5} & {54.6} & {19.8} & {2.9} & {27.5} & 17.2 & {24.4} \\ 
$\quad \llcorner$ \textbf{FABSVer-1.5B} & \textbf{81.0} & \textbf{81.9} & \textbf{91.7} & \textbf{83.1} & {61.4} & \textbf{79.8} & \textbf{56.8} & \textbf{22.4} & \textbf{4.6} & \textbf{30.6} & \textbf{18.7} & \textbf{26.6} \\ 
\midrule
\multicolumn{7}{l|}{\em Qwen2.5-3B} \\
$\quad \vdash$ Base & 39.5 & 21.2 & 13.6 & 22.5 & 29.9 & 25.3 & 32.7 & 10.7 & 1.3 & 15.3 & 10.3 & 14.1 \\
$\quad \vdash$ Instruct & 65.6 & 35.0 & 21.0 & 45.5 & 37.6 & 40.9 & 61.0 & 24.6 & 3.8 & 34.1 & 25.6 & 29.8 \\
$\quad \vdash$ SFT & 21.5 & 13.2 & 2.1 & 10.9 & 17.9 & 13.1 & 14.4 & 2.8 & 0.4 & 5.3 & 2.9 & 5.2 \\
$\quad \vdash$ Zero-RL & 64.9 & 31.2 & 13.0 & 39.7 & 30.3 & 35.8 & 64.2 & {26.6} & 6.7 & 37.5 & {27.4} & 32.5 \\
$\quad \vdash$ {RISE-3B} & {81.0} & {73.6} & \textbf{86.3} & {74.4} & \textbf{56.1} & {74.3} & {64.3} & {26.6} & \textbf{7.9} & \textbf{42.5} & 26.2 & {33.5} \\ 

$\quad \llcorner$ \textbf{FABSVer-3B} & \textbf{85.2} & \textbf{77.8} & {85.4} & \textbf{81.3} & {55.1} & \textbf{77.0} & \textbf{67.5} & \textbf{30.0} & {6.3} & {37.5} & \textbf{28.0} & \textbf{33.9} \\ 

\midrule
\multicolumn{7}{l|}{\em Qwen2.5-7B} \\
$\quad \vdash$ Base & 58.4 & 48.4 & 45.9 & 51.5 & 48.4 & 50.5 & 38.3 & 13.2 & 2.1 & 21.9 & 11.9 & 17.5 \\
$\quad \vdash$ Instruct & 77.2 & 45.2 & 26.3 & 57.0 & 40.2 & 49.2 & 73.8 & 35.8 & 10.0 & 50.6 & \textbf{35.9} & 41.2 \\
$\quad \vdash$ SFT & 40.5 & 36.1 & 36.6 & 47.4 & 39.2 & 40.0 & 28.7 & 7.2 & 0.8 & 13.8 & 6.2 & 11.3 \\
$\quad \vdash$ Zero-RL & 75.9 & 41.6 & 21.7 & 56.5 & 37.3 & 46.6 & 74.5 & {36.7} & 12.1 & 51.3 & 34.2 & 41.7 \\
$\quad \vdash$ {RISE-7B} & {83.8} & {65.9} & {75.0} & {72.5} & \textbf{48.6} & {69.2} & {74.8} & {36.7} & {12.5} & \textbf{55.9} & 34.6 & {42.9} \\ 

$\quad \llcorner$ \textbf{FABSVer-7B} & \textbf{88.3} & \textbf{70.1} & \textbf{85.8} & \textbf{76.9} & {47.5} & \textbf{73.7} & \textbf{77.4} & 

\textbf{41.0} & \textbf{13.8} & {53.1} & {35.4} & \textbf{44.1} \\ 

\bottomrule

\end{tabular}
}
\end{table*}

\begin{table}[htbp]
\centering
\caption{Training time (GPU hours) of Fabvers and RISE on H20 GPUs.}
\label{tab:training_time}
\begin{tabular}{lcccccc}
\toprule
& \multicolumn{2}{c}{Qwen2.5-1.5B Base} & \multicolumn{2}{c}{Qwen2.5-3B Base} & \multicolumn{2}{c}{Qwen2.5-7B Base} \\
\cmidrule(lr){2-3} \cmidrule(lr){4-5} \cmidrule(lr){6-7}
& Fabvers & RISE & Fabvers & RISE & Fabvers & RISE \\
\midrule
GPU hours & 98 & 192 & 136 & 192 & 244 & 384 \\
\bottomrule
\end{tabular}
\end{table}

\subsection{Ablation Study}

To evaluate the effectiveness of the proposed Dynamic Reference Model Update (DRMU), we conduct ablation experiments on FABSVer by removing DRMU while keeping all other components unchanged. Both quantitative results on math benchmarks and training dynamics are analyzed.

\paragraph{Quantitative Results.}
Table~\ref{tab:ablation} presents the ablation results across 1.5B, 3B, and 7B model sizes. Overall, DRMU brings consistent improvements in both self-verification and reasoning capabilities. In terms of self-verification, FABSVer with DRMU outperforms its counterpart without DRMU by 1.7\%, 6.2\%, and 2.6\% on average for the 1.5B, 3B, and 7B models, respectively. The gains are particularly pronounced on challenging competition-level benchmarks such as AIME and OlympiadBench, demonstrating that DRMU effectively enhances the model's error-detection ability on difficult problems. In terms of reasoning, DRMU also yields consistent improvements across all three model sizes,  confirming that the enhanced self-verification capability does not come at the cost of reasoning performance.

\paragraph{Training Dynamics.}
Figure~\ref{fig:drmu_effect} illustrates the training curves of combined reward and KL divergence for FABSVer with and without DRMU. For the variant with DRMU, the reference model is updated at steps 273, 546, and 819, which correspond to epochs 4, 8, and 12, respectively. After each update, the KL divergence drops sharply to approximately zero, since the current policy becomes the new reference. This \textbf{KL reset} releases the constraint imposed by the fixed reference model, allowing the policy to continue exploring and improving its reward. As a result, FABSVer with DRMU achieves higher final rewards across all three model sizes compared to the variant without DRMU.

In contrast, without DRMU, the KL divergence ceases to exhibit significant growth in the later stages of training, and the reward improvement decelerates accordingly. This observation aligns with \Cref{thm:drmu}. Consequently, when $\mathrm{KL}[\pi_\theta \| \pi_{\text{ref}}]$ stops increasing, the ceiling on the achievable reward is effectively capped, limiting the potential for further improvement. By periodically resetting the KL to zero, DRMU raises this ceiling and thus enables continued reward growth.

\begin{table*}[t]
\centering
\caption{Ablation results of FABSVer with and without DRMU.}
\label{tab:ablation}
\resizebox{1.0\linewidth}{!}{
\setlength{\tabcolsep}{3pt}
\begin{tabular}{lrrrrrr|rrrrrr}
\toprule
& \multicolumn{6}{c|}{\textit{\textbf{Self-Verification}}} & \multicolumn{6}{c}{\textit{\textbf{Reasoning}}} \\
\cmidrule(lr){2-7} \cmidrule(lr){8-13}
\textbf{Model} & \textbf{\small MATH} & \textbf{\small Olym.} & \textbf{\small AIME} & \textbf{\small AMC} & \textbf{\small Mine.} & \textbf{Avg.} & \textbf{\small MATH} & \textbf{\small Olym.} & \textbf{\small AIME} & \textbf{\small AMC} & \textbf{\small Mine.} & \textbf{Avg.} \\
\midrule

\multicolumn{7}{l|}{\em FABSVer-1.5B} \\

$\quad \vdash$ \;w/\phantom{o} \;  DRMU  & \textbf{81.0} & \textbf{81.9} & \textbf{91.7} & \textbf{83.1} & \textbf{61.4} & \textbf{79.8} & \textbf{56.8} & \textbf{22.4} & \textbf{4.6} & \textbf{30.6} & \textbf{18.7} & \textbf{26.6} \\ 

$\quad \llcorner$ \; w/o \; DRMU & 80.1 & {81.4} & {89.8} & {79.1} & {60.3} & {78.1} & {56.4} & {21.6} & {3.8} & {27.8} & {18.2} & {25.6} \\ 
\midrule

\multicolumn{7}{l|}{\em  FABSVer-3B} \\
$\quad \vdash$ \;w/\phantom{o} \;  DRMU  & \textbf{85.2} & \textbf{77.8} & \textbf{85.4} & \textbf{81.3} & \textbf{55.1} & \textbf{77.0} & \textbf{67.5} & \textbf{30.0} & \textbf{6.3} & {37.5} & \textbf{28.0} & \textbf{33.9} \\ 

$\quad \llcorner$ \; w/o \; DRMU & {81.1} & {70.3} & {82.5} & {70.6} & {49.3} & {70.8} & {64.8} & {28.8} & {5.3} & \textbf{40.0} & {27.8} & {33.3} \\ 
\midrule

\multicolumn{7}{l|}{\em FABSVer-7B} \\

$\quad \vdash$ \;w/\phantom{o} \;  DRMU  &  \textbf{88.3} & \textbf{70.1} & \textbf{85.8} & \textbf{76.9} & \textbf{47.5} & \textbf{73.7} & \textbf{77.4} & \textbf{41.0} & \textbf{13.8} & \textbf{53.1} & \textbf{35.4} & \textbf{44.1} \\ 

$\quad \llcorner$ \; w/o \; DRMU & {87.3} & {69.4} & {80.4} & {74.3} & {43.9} & {71.1} & {76.1} & {38.7} & \textbf{13.8} & {49.7} & {34.7} & {42.6} \\ 

\bottomrule

\end{tabular}
}
\end{table*}

\begin{figure}
    \centering
    \includegraphics[width=1.0\textwidth]{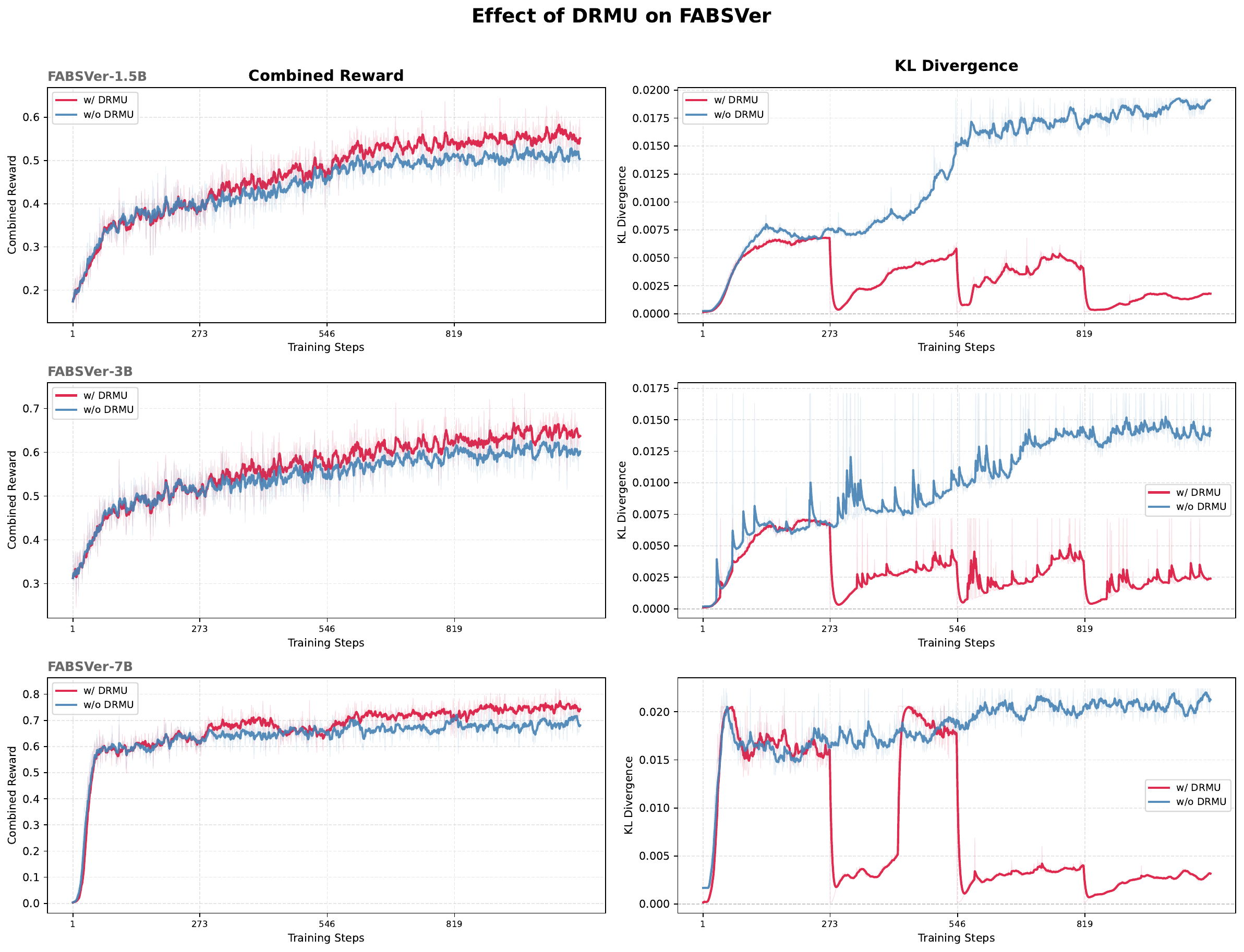}
    \caption{The effect of DRMU on FABSVer. For FABSVer with DRMU, steps 273, 546, and 819 correspond to the updates of the reference model using the current policy. After each update, the KL divergence drops sharply to approximately zero, since the current policy becomes the new reference.}
    \label{fig:drmu_effect}
\end{figure}

\subsection{Analysis of Answer and Verify Rewards during Training}
Figure~\ref{fig:reward_components} presents the training curves of answer reward and verify reward for FABSVer across three model sizes. Several notable patterns emerge from these dynamics.

In the early stages of training, the answer reward and verify reward are nearly identical across all model sizes. We observe that the model exhibits a strong tendency to always predict its own outputs as correct at this stage. This can be attributed to a \textbf{free-riding effect}: as the model's answer generation capability steadily improves, it produces correct answers with increasing frequency. Consequently, the verify task can achieve a steadily increasing reward by simply predicting ``Correct'' for every sample, without developing any genuine discriminative ability. As a result, the verify reward passively tracks the answer reward, and the two remain tightly coupled.

As training progresses into the middle and later stages, the growth of answer reward decelerates. At this point, the free-riding strategy becomes insufficient---merely judging most outputs as correct no longer yields additional reward gains. To further improve the overall reward, the model is forced to develop the ability to correctly identify its own errors. This shift from a naive ``almost-always-correct'' strategy to a more sophisticated discriminative policy drives a second wave of verify reward improvement, allowing it to outpace answer reward. This phenomenon represents a \textbf{forced learning effect}: the model must learn to discriminate in order to unlock further reward when answer quality plateaus.

Notably, the gap between verify reward and answer reward shrinks as model size increases. For smaller models (e.g., 1.5B), verify reward notably surpasses answer reward, reflecting their greater reliance on explicit verification learning. For larger models (e.g., 7B), the gap between the two rewards is considerably smaller, suggesting that larger models possess stronger intrinsic self-verification capabilities---likely acquired during pretraining---and thus require less explicit discriminative training to achieve relatively balanced performance. This observation is consistent with the results in Table~\ref{tab:zero_shot_results}, where larger models consistently achieve higher absolute self-verification scores.

\begin{figure}[t]
    \centering
    \includegraphics[width=1.0\textwidth]{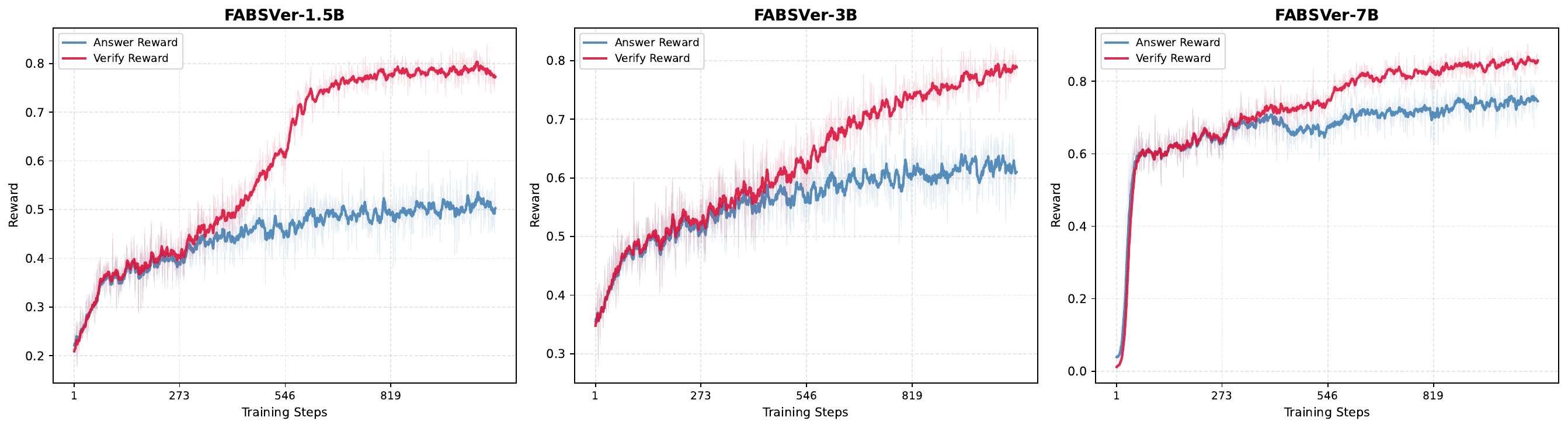}
    \caption{Training curves of answer reward and verify reward for FABSVer across different model sizes.}
    \label{fig:reward_components}
\end{figure}

\section{Conclusion}
In this paper, we propose FABSVer, a streamlined training paradigm that fuses solution generation and self-verification into a single forward pass, substantially reducing training overhead while achieving synergistic gains in both capabilities. Through a combined reward design and the Dynamic Reference Model Update (DRMU) mechanism, FABSVer overcomes the convergence bottleneck imposed by a fixed reference model, enabling sustained reward growth. Our experimental results demonstrate that FABSVer consistently outperforms existing methods across multiple model sizes and math benchmarks, achieving stronger performance with lower computational cost. Further analysis of training dynamics reveals that models undergo a two-phase transition in acquiring self-verification ability—from an early free-riding stage to a later forced learning stage.

\bibliographystyle{plainnat}
\bibliography{bib}

\newpage

\appendix

\section{Limitation}
\label{sec:limit}

The trained LLM, although capable of serving as a rule verifier to guide RL training on unlabeled data for self-improvement, is not explored in this direction within this paper and remains future work. Meanwhile, the training data in this work are entirely derived from mathematical reasoning tasks; despite the model exhibiting a certain degree of generalization ability, its effectiveness in broader domains beyond those with verifiable rewards still requires further validation. In addition, we primarily adopt final inference accuracy as the evaluation metric. However, a single metric makes it difficult to fully capture the criticality of the model's emergent verification behavior, indicating that the evaluation system still needs further refinement.

\section{Training Details.}
\label{sec:train-details}
 We use \texttt{trl} \citep{trl} for our experiments. During RL training, the training batch size is set to 1024, with a constant learning rate of 5e-7 and a uniform KL penalty coefficient of 1e-2. The sampling temperature is set to 1.0, and we roll out 8 responses per problem. The maximum output length is set to 3000 tokens. We train for 16 epochs, and the reference model is updated every 4 epochs. As the Qwen2.5-1.5B Base and Qwen2.5-3B Base models have difficulty producing responses in the required format, we first sample 5K correct responses from the Qwen2.5-7B Base model's training run and use them to fine-tune the 1.5B and 3B models for 20 steps, before proceeding with RL training on these two models.

\section{Algorithm Overview}
\label{sec:alg-overview}
\begin{algorithm}[]

\caption{FABSVer (GRPO)}
\label{alg:fabsver}

\textbf{Input} Language model $\pi_{\theta_{\text{init}}}$; dataset $\mathcal{D}$; number of solutions per question $K$; generation batch size $\mathcal{B}_g$; total training iterations $N$; reference model update interval $U$.

\begin{algorithmic}[1]
  \State \textbf{Initialize:} policy $\pi_{\theta}\leftarrow\pi_{\theta_{\text{init}}}$, reference model $\pi_{\text{ref}}\leftarrow\pi_{\theta_{\text{init}}}$

  \For{$\text{iteration}=1$ \textbf{to} $N$}
    \State Sample $\mathcal{B}_g$ questions $\{(\mathbf{x}_i, \mathbf{y}_i^*)\}_{i=1}^{\mathcal{B}_g} \sim \mathcal{D}$
    
    \State Generate solutions and verification results in a single forward pass: \Comment{Task Fusion}
    \Statex \hspace{\algorithmicindent} $\displaystyle 
           \mathcal{B}_{\text{train}}\gets
           \bigl\{(\mathbf{x}_i, \mathbf{y}_i^{(k)}, \mathbf{v}_i^{(k)}, r_{\text{combine}}(\mathbf{y}_i^{(k)}, \mathbf{v}_i^{(k)}, \mathbf{y}_i^*))
           \,\bigm|\,
           (\mathbf{y}_i^{(k)}, \mathbf{v}_i^{(k)}) \sim \pi_\theta(\cdot|\mathbf{x}_i),
           \,i \le \mathcal{B}_g,\, k \le K \bigr\}$
    
    \State Compute rewards using Eq.~\ref{eq:combine}
    
    \State Update policy $\pi_{\theta}$ using GRPO objective (Eq.~\ref{eq:GRPO}) with reference model $\pi_{\text{ref}}$
    
    \If{$\text{iteration} \bmod U = 0$} \Comment{Dynamic Reference Model Update}
      \State $\pi_{\text{ref}} \leftarrow \pi_{\theta}$
    \EndIf
  \EndFor
\end{algorithmic}
\textbf{Output} Optimized policy $\pi_\theta$

\end{algorithm}

\newpage

\section{Prompt for training and inference}

\begin{figure}[h]
\centering
\begin{promptspecial}
<|im\_start|>system \\
A conversation between User and Assistant. The user asks a question, and the Assistant solves it. The assistant first thinks about the reasoning process in the mind and then provides the user with the answer along with the verification result of the answer. The reasoning process, answer and verification result are enclosed within <think> </think>, <answer> </answer> and <verify> </verify> tags, respectively, i.e., <think>\\ please provide the detailed reasoning process here \\</think>\\\\<answer>\\ please provide the complete answering process here, then output the final answer in \textbackslash \textbackslash boxed\{Your answer\} \\</answer>\\\\<verify>\\ please provide the detailed verification process here, then output \textbackslash \textbackslash boxed\{Correct\} if the answer is correct, otherwise output \textbackslash \textbackslash boxed\{Incorrect\} \\</verify>
<|im\_end|> \\
<|im\_start|>user \\
\{\textbf{Input}\}<|im\_end|> \\
<|im\_start|>assistant
\end{promptspecial}
\caption{Prompt template used in the training and inference.}
\label{fig:prompt}
\end{figure}

% \begin{figure}[h]
% \centering
% \begin{promptspecial}
% <|im\_start|>system \\
% You are a helpful assistant.<|im\_end|> \\
% <|im\_start|>user \\
% \{\textbf{Input}\} Please reason step by step, and put your final answer within \textbackslash \textbackslash boxed\{\}.<|im\_end|> \\
% <|im\_start|>assistant
% \end{promptspecial}
% \caption{Prompt template used for Qwen base model evaluation.}
% \label{fig:base_eval_prompt}
% \end{figure}

\newpage

\section{Verification Prompt in Related Approaches}

\begin{figure}[h]
\centering
\begin{prompttext}
Below you are presented with a question and a tentative response. Your task is to evaluate and assign a rating to the response based on the following clear criteria:

\vspace{10pt} 
Rating Criteria:

\vspace{10pt}
1. Missing final answer enclosed in \textbackslash \textbackslash boxed\{\} at the end: assign \textbackslash \textbackslash boxed\{-1\}.

2. Correct response with the final answer enclosed in \textbackslash \textbackslash boxed\{\} at the end: assign \textbackslash \textbackslash boxed\{1\}.

3. Incorrect response with the final answer enclosed in \textbackslash \textbackslash boxed{} at the end: assign \textbackslash \textbackslash boxed\{-0.5\}.

\vspace{10pt} 
\#\#\# Question Begin \#\#\#

\{\textbf{Question}\}

\#\#\# Question End \#\#\# 

\vspace{10pt} 
\#\#\# Response Begin \#\#\#

\{\textbf{Response}\}

\#\#\# Response End \#\#\#

\vspace{10pt} 
Briefly summarize your analysis, then clearly state your final rating value enclosed in \textbackslash \textbackslash boxed\{\} at the end.
\end{prompttext}
\caption{Verification prompt used in \citep{rise}.}
\label{fig:rise_ver_prompt}
\end{figure}

\begin{figure}[h]
\centering
\begin{prompttext}

\vspace{10pt} 

User: [Math problem]

\vspace{10pt} 

Assistant: [Solution]

\vspace{10pt} 

User : Please verify the solution step by step. At the end of the solution verification, when you give your final grade, write it in the form " Is the answer correct ( Yes / No ) ? X" ,
where X is either Yes or No.

\vspace{10pt} 

Assistant :

\end{prompttext}
\caption{Verification prompt used in \citep{zhang2025incentivizing}.}
\label{fig:self_ver_prompt}
\end{figure}

 % output

\newpage

\section{Proof of \cref{thm:drmu}}
\label{proof}
\thmOne*
\begin{proof}

We can construct a Gibbs distribution \(\pi_{\phi}(y|x)\) based on \(\pi_{\text{ref}}(y|x)\) and \(r(x,y)\) as follows:
\begin{align}
    \pi_{\phi}(y|x) = \frac{e^{r(x,y)}}{\mathbb{E}_{\pi_{\text{ref}}(y|x)}[e^{r(x,y)}] }\pi_{\text{ref}}(y|x).
\end{align}

Furthermore, we can obtain:
\begin{align}
\label{eq:parb_ratio}
    \frac{\pi_{\text{ref}}(y|x)}{\pi_{\phi}(y|x) }  = \frac{\mathbb{E}_{\pi_{\text{ref}}(y|x)}[e^{r(x,y)}] }{e^{r(x,y)}}.
\end{align}

Now, let us compute the KL divergence between \(\pi_{\theta}(y|x)\) and \(\pi_{\phi}(y|x)\):

\begin{align}
\text{KL}[\pi_{\theta} \| \pi_{\phi}]  &= \mathbb{E}_{\pi_\theta(y|x)}\left[\log \frac{\pi_\theta(y|x)}{\pi_\phi(y|x)}\right] = \mathbb{E}_{\pi_\theta(y|x)}\left[\log \left( \frac{\pi_\theta(y|x)}{\pi_\text{ref}(y|x)} \cdot \frac{\pi_\text{ref}(y|x)}{\pi_\phi(y|x)} \right) \right]  \\
&= \text{KL}[\pi_{\theta} \| \pi_{\text{ref}}]  + \mathbb{E}_{\pi_\theta(y|x)}\left[\log \frac{\pi_\text{ref}(y|x)}{\pi_\phi(y|x)}\right].
\end{align}

Substituting Eq.~\ref{eq:parb_ratio} into the above expression yields:
\begin{align}
\text{KL}[\pi_{\theta} \| \pi_{\phi}] 
&= \text{KL}[\pi_{\theta} \| \pi_{\text{ref}}]  + \mathbb{E}_{\pi_\theta(y|x)}\left[\log \frac{\mathbb{E}_{\pi_{\text{ref}}(y|x)}[e^{r(x,y)}] }{e^{r(x,y)}}\right]  \\
&= \text{KL}[\pi_{\theta} \| \pi_{\text{ref}}]  + \log \mathbb{E}_{\pi_{\text{ref}}(y|x)}[e^{r(x,y)}]  - \mathbb{E}_{\pi_\theta(y|x)}[r(x,y)].
\end{align}

By the non-negativity of KL divergence, we have \(\text{KL}[\pi_{\theta} \| \pi_{\phi}] \ge 0\). Therefore, we can finally prove:
\begin{align}
    \mathbb{E}_{\pi_\theta(y|x)}[r(x,y)] \le \text{KL}[\pi_{\theta} \| \pi_{\text{ref}}] + \log \mathbb{E}_{\pi_{\text{ref}}(y|x)}[e^{r(x,y)}].
\end{align}

\end{proof}

\newpage
\section{Examples of Outputs}

% 1.5B
\begin{figure}[h]
\centering
\begin{outputcase}

\vspace{10pt} 

\textbf{Input}: Let \[f(n) = \left\{ \begin{array}{cl}
n^3+2n-1 &\text{ if }n>1, \\
n-1 &\text{ if }n \le 1.
\end{array}
\right.\] Find $f(0)+f(1)+f(2)$.

\vspace{10pt} 

\textbf{Output}: 

<think>

Let's evaluate \( f(0) \), \( f(1) \), and \( f(2) \) and sum them up. 

1. For \( f(0) \):
\[ f(0) = \left\{
\begin{array}{cl}
0^3 + 2 \cdot 0 - 1 & \text{if } 0 > 1, \\
0 - 1 & \text{if } 0 \le 1.
\end{array}
\right.
\]
Since \( 0 \le 1 \), we use the second expression:
\[ f(0) = 0 - 1 = -1. \]

2. For \( f(1) \):
\[ f(1) = \left\{
\begin{array}{cl}
1^3 + 2 \cdot 1 - 1 & \text{if } 1 > 1, \\
1 - 1 & \text{if } 1 \le 1.
\end{array}
\right.
\]
Since \( 1 \le 1 \), we use the second expression:
\[ f(1) = 1 - 1 = 0. \]

3. For \( f(2) \):
\[ f(2) = 2^3 + 2 \cdot 2 - 1 = 8 + 4 - 1 = 11. \]

Now we sum these values:
\[ f(0) + f(1) + f(2) = -1 + 0 + 11 = 10. \]

</think>

<answer>

\textbackslash boxed\{10\}

</answer>

<verify>

To verify the calculation, I will go back to the definitions and the rules I stated.

1. \( f(0) = -1 \)
2. \( f(1) = 0 \)
3. \( f(2) = 11 \)

Summing them up:
\[ -1 + 0 + 11 = 10 \]

Verifying the correctness, I get \textbackslash boxed\{Correct\}.

</verify>

\vspace{10pt} 

\end{outputcase}
\caption{Output case of FABSVer-1.5B.}
\label{fig:ver_prompt}
\end{figure}

% 3B
\begin{figure}[h]
\centering
\begin{outputcase}

\vspace{10pt} 

\textbf{Input}: How many ways are there to put 4 distinguishable balls 
into 2 distinguishable boxes?

\vspace{10pt} 

\textbf{Output}: 

\textless think\textgreater

To solve this problem, we can use the concept of permutations. 
Since each ball must go into one of two distinguishable boxes, 
and the balls are distinguishable, each ball has two choices 
(either of the two boxes).
\vspace{10pt} 

For 4 distinguishable balls, each ball has 2 choices:
(For the first ball) 2 choices, (For the second ball) 2 choices, 
(For the third ball) 2 choices, and (For the fourth ball) 2 choices.

\vspace{10pt} 
The total number of ways to distribute the balls is 
\(2 \times 2 \times 2 \times 2 = 2^4\).

Calculating this, 

\[2^4 = 16\]

Therefore, there are 16 different ways to put 4 distinguishable 
balls into 2 distinguishable boxes.

\textless /think\textgreater

\textless answer\textgreater

\textbackslash boxed\{16\}

\textless /answer\textgreater

\textless verify\textgreater

The problem requires counting the number of ways to distribute 
4 distinguishable balls into 2 distinguishable boxes. 

\vspace{10pt} 

We established that each of the 4 balls has 2 choices, 
leading to \(2^4 = 16\) possible distributions.

\vspace{10pt} 

Each distribution is indeed one arrangement of the 4 balls in 
two boxes, which confirms our answer. Therefore, the answer is \textbackslash boxed\{Correct\}.

\textless /verify\textgreater

\vspace{10pt} 

\end{outputcase}
\caption{Output case of FABSVer-3B.}
\label{fig:ver_prompt}
\end{figure}

%
% 7B
\begin{figure}[h]
\centering
\begin{outputcase}

\vspace{10pt}

\textbf{Input}: What is the value of \(y\) in the arithmetic sequence \(y + 6\), \(12\), \(y\)?

\textbf{Output}:

\textless think\textgreater

To determine the value of \(y\) in the arithmetic sequence, we need to use the property of arithmetic sequences that the difference between consecutive terms is constant. Let's denote the common difference of the sequence as \(d\).

\vspace{10pt}

From the sequence \(y + 6, 12, y\):

\vspace{10pt}

1. The difference between the first and second term should be equal to the difference between the third and second term:
\[ 12 - (y + 6) = y - 12 \]

\vspace{10pt}

2. Simplify both sides of the equation:
\[ 12 - y - 6 = y - 12 \]
\[ 6 - y = y - 12 \]

\vspace{10pt}

3. Solve for \(y\):
\[ 6 + 12 = y + y \]
\[ 18 = 2y \]
\[ y = 9 \]

\textless /think\textgreater

\textless answer\textgreater

The value of \(y\) is \textbackslash boxed\{9\}.

\textless /answer\textgreater

\textless verify\textgreater

\vspace{10pt}

Let's verify the solution by checking if it maintains the same common difference \(d\):

\vspace{10pt}

- The difference between the first and second term with \(y = 9\):
\[ 12 - (9 + 6) = 12 - 15 = -3 \]

\vspace{10pt}

- The difference between the second and third term with \(y = 9\):
\[ 9 - 12 = -3 \]

\vspace{10pt}

Since both differences are equal, the solution is \textbackslash boxed\{Correct\}

\textless /verify\textgreater

\vspace{10pt}

\end{outputcase}
\caption{Output case of FABSVer-7B.}
\label{fig:ver_prompt}
\end{figure}

\end{document}